\documentclass{article}

\usepackage{graphicx} 
\usepackage{subcaption} 

\usepackage[space]{grffile}

\usepackage{color}
\usepackage{amsmath}
\usepackage{amsthm}
\usepackage{float}
\usepackage[titletoc,title]{appendix}
\usepackage{amsfonts}
\usepackage{thmtools,thm-restate}

\usepackage{natbib}

\usepackage{algorithm}
\usepackage{algorithmic}

\usepackage{hyperref}

\newcommand{\beginsupplement}{%
        \setcounter{table}{0}
        \renewcommand{\thetable}{S\arabic{table}}%
        \setcounter{figure}{0}
        \renewcommand{\thefigure}{S\arabic{figure}}%
     }

\newif\ifpaper

\papertrue

\ifpaper
    \usepackage{times}
    \usepackage[accepted]{icml2016}
    \icmltitlerunning{Interactive Bayesian Hierarchical Clustering}

\else
    \usepackage[margin=1in]{geometry}
    \title{Interactive Bayesian Hierarchical Clustering}
    
    \author{
    Sharad Vikram \& Sanjoy Dasgupta\\ 
    Computer Science and Engineering \\ 
    University of California, San Diego \\ 
    \texttt{\{svikram, dasgupta\}@cs.ucsd.edu}
    }
\fi

\begin{document} 
\ifpaper
    \twocolumn[
    \icmltitle{Interactive Bayesian Hierarchical Clustering}
   
    \icmlauthor{Sharad Vikram}{svikram@cs.ucsd.edu}
    \icmlauthor{Sanjoy Dasgupta}{dasgupta@cs.ucsd.edu}
    \icmladdress{Computer Science and Engineering, UCSD,
                9500 Gilman Drive, La Jolla, CA 92093}
    \icmlkeywords{machine learning}
    
    \vskip 0.3in
    ]
\else
    \maketitle
\fi

\begin{abstract}
Clustering is a powerful tool in data analysis, but it is often difficult to 
find a grouping that aligns with a user's needs.
To address this, several methods incorporate 
constraints
obtained from users into clustering algorithms,
but unfortunately do not apply to hierarchical clustering.
We design an \emph{interactive Bayesian} algorithm that incorporates user interaction
into hierarchical clustering while still utilizing
the geometry of the data by
sampling a constrained posterior
distribution over hierarchies. 
We also suggest several ways
to intelligently query a user. 
The algorithm, along with the querying schemes, shows promising
results on real data.
\end{abstract}

\section{Introduction}

{\it Clustering} is a basic tool of exploratory data analysis. There are a variety of efficient algorithms---including $k$-means, EM for Gaussian mixtures, and hierarchical agglomerative schemes---that are widely used for discovering ``natural'' groups in data. Unfortunately, they don't always find a grouping that suits the user's needs.


This is inevitable. In any moderately complex data set, there are many different plausible grouping criteria. Should a collection of rocks be grouped according to value, or shininess, or geological properties? Should animal pictures be grouped according to the Linnaean taxonomy, or cuteness? Different users have different priorities, and an unsupervised algorithm has no way to magically guess these.

As a result, a rich body of work on {\it constrained} clustering has emerged.
In this setting, a user supplies guidance, typically in the form of ``must-link'' or ``cannot-link'' constraints, pairs of points that must be placed together or apart. Introduced by \citet{WC00}, these constraints have since been incorporated into many different {\it flat} clustering procedures~\citep{WCRS01,BBC04,BBM04,KBDM05,BJ14}.

In this paper, we introduce constraints to {\it hierarchical clustering}, the recursive partitioning of a data set into successively smaller clusters to form a tree. A hierarchy has several advantages over a flat clustering. First, there is no need to specify the number of clusters in advance. Second, the tree captures cluster structure at multiple levels of granularity, simultaneously. As such, trees are particularly well-suited for exploratory data analysis and the discovery of natural groups.

There are several well-established methods for hierarchical clustering, the most prominent among which are the bottom-up agglomerative methods such as average linkage (see, for instance, Chapter 14 of \citet{HTF09}). But they suffer from the same problem of under-specification that is the scourge of unsupervised learning in general. And, despite the rich literature on incorporating additional guidance into flat clustering, there has been relatively little work on the hierarchical case.

What form might the user's guidance take? The usual must-link and cannot-link constraints make little sense when data has hierarchical structure. Among living creatures, for instance, should {\tt elephant} and {\tt tiger} be linked? At some level, yes, but at a finer level, no. A more straightforward assertion is that {\tt elephant} and {\tt tiger} should be linked in a cluster that does not include {\tt snake}. We can write this as a {\it triplet} $(\{\mbox{\tt elephant}, \mbox{\tt tiger}\}, \mbox{\tt snake})$. We could also assert $(\{\mbox{\tt tiger}, \mbox{\tt leopard}\}, \mbox{\tt elephant})$. Formally, $(\{a,b\},c)$ stipulates that the hierarchy contains a subtree (that is, a cluster) containing $a$ and $b$ but not $c$.

A wealth of research addresses learning taxonomies from triplets {\it alone}, mostly in the field of phylogenetics: see \citet{F04} for an overview, and \citet{ASSU81} for a central algorithmic result. Let's say there are $n$ data items to be clustered, and that the user seeks a particular hierarchy $T^*$ on these items. This $T^*$ embodies at most ${n \choose 3}$ triplet constraints, possibly less if it is not binary. It was pointed out in \citet{TLBSK11} that roughly $n \log n$ carefully-chosen triplets are enough to fully specify $T^*$ if it is balanced. This is also a lower bound: there are $n^{\Omega(n)}$ different labeled rooted trees, so each tree requires $\Omega(n \log n)$ bits, on average, to write down---and each triple provides $O(1)$ bits of information, since there are just three possible outcomes for each set of points $a,b,c$. Although $n \log n$ is a big improvement over $n^3$, it is impractical for a user to provide this much guidance when the number of points is large. In such cases, a hierarchical clustering cannot be obtained on the basis of constraints alone; the geometry of the data must play a role.

We consider an interactive process during which a user incrementally adds constraints.
\begin{itemize}
\item Starting with a pool of data $X \subseteq \mathbb{R}^d$, the machine builds a candidate hierarchy $T$.
\item The set of constraints $C$ is initially empty.
\item Repeat:
\begin{itemize}
\item The machine presents the user with a small portion of $T$: specifically, its restriction to $O(1)$ leaves $S \subset X$. We denote this $T |_S$.
\item The user either accepts $T|_S$, or provides a triplet constraint $(\{a,b\},c)$ that is violated by it.
\item If a triplet is provided, the machine adds it to $C$ and modifies the tree $T$ accordingly.
\end{itemize}
\end{itemize}
In realizing this scheme, a suitable clustering algorithm and querying strategy must be designed. Similar issues have been confronted in flat clustering---with must-link and cannot-link constraints---but the solutions are unsuitable for hierarchies, and thus a fresh treatment is warranted.

\subsubsection*{The clustering algorithm}

What is a method of hierarchical clustering that takes into account the geometry of the data points as well as user-imposed constraints? 

We adopt an \emph{interactive Bayesian} approach. The learning procedure is uncertain about the intended tree and this uncertainty is captured in the form of a distribution over all possible trees. Initially, this distribution is informed solely by the geometry of the data but once interaction begins, it is also shaped by the growing set of constraints.

The nonparametric Bayes literature contains a variety of different distributional models for hierarchical clustering. We describe a general methodology for extending these to incorporate user-specified constraints. For concreteness, we focus on the Dirichlet diffusion tree~\citep{N03}, which has enjoyed empirical success. We show that triplet constraints are quite easily accommodated: when using a Metropolis-Hastings sampler, they can efficiently be enforced, and the state space remains strongly connected, assuring convergence to the unique stationary distribution.

\subsubsection*{The querying strategy}

What is a good way to select the subsets $S$? A simple option is to pick them at random from $X$. We show that this strategy leads to convergence to the target tree $T^*$. Along the way, we define a suitable distance function for measuring how close $T$ is to $T^*$. 


We might hope, however, that a more careful choice of $S$ would lead to faster convergence, in much the same way that intelligent querying is often superior to random querying in active learning. In order to do this, we show how the Bayesian framework allows us to quantify which portions of the tree are the most uncertain, and thereby to pick $S$ that focuses on these regions.

Querying based on uncertainty sounds promising, but is dangerous because it is heavily influenced by the choice of prior, which is ultimately quite arbitrary. Indeed, if only such queries were used, the interactive learning process could easily converge to the wrong tree. We show how to avoid this situation by interleaving the two types of queries.

Finally, we present a series of experiments that illustrate how a little interaction leads to significantly better hierarchical clusterings.

\subsection{Other related work}

A related problem that has been studied in more detail~\citep{ZB00,EDSN11,KBXS12} is that of building a hierarchical clustering where the only information available is pairwise similarities between points, but these are initially hidden and must be individually queried.

In another variant of interactive flat clustering~\citep{BB08,AZ10,ABV14}, the user is allowed to specify that individual clusters be merged or split. A succession of such operations can always lead to a target clustering, and a question of interest is how quickly this convergence can be achieved.

Finally, it is worth mentioning the use of triplet constraints in learning other structures, such as Euclidean embeddings~\citep{BG05}.

\section{Bayesian hierarchical clustering}

The most basic form of hierarchical clustering is a rooted binary tree with the data points at its leaves. This is sometimes called a {\it cladogram}. Very often, however, the tree is adorned with additional information, for instance:
\begin{enumerate}
\item An ordering of the internal nodes, where the root is assigned the lowest number and each node has a higher number than its parent.

This ordering uniquely specifies the induced $k$-clustering (for any $k$): just remove the $k-1$ lowest-numbered nodes and take the clusters to be the leaf-sets of the $k$ resulting subtrees.

\item Lengths on the edges.

Intuitively, these lengths correspond to the amount of change (for instance, time elapsed) along the corresponding edges. They induce a {\it tree metric} on the nodes, and often, the leaves are required to be at the same distance from the root.

\item Parameters at internal nodes.

These parameters are sometimes from the same space as the data, representing intermediate values on the way from the root to the leaves.
\end{enumerate}
Many generative processes for trees end up producing these more sophisticated structures, with the understanding that undesired additional information can simply be discarded at the very end. We now review some well-known distributions over trees and over hierarchical clusterings.

Let's start with cladograms on $n$ leaves. The simplest distribution over these is the uniform. Another well-studied option is the {\it Yule model}, which can be described using either a top-down or bottom-up generative process. The top-down view corresponds to a continuous-time pure birth process: start with one lineage; each lineage persists for a random exponential(1) amount of time and then splits into two lineages; this goes on until there are $n$ lineages. The bottom-up view is a coalescing process: start with $n$ points; pick a random pair of them to merge; then repeat. \citet{A95} has defined a one-parameter family of distributions over cladograms, called the {\it beta-splitting model}, that includes the uniform and the Yule model as special cases. It is a top-down generative process in which, roughly, each split is made by sampling from a Beta distribution to decide how many points go on each side. To move to arbitrary (not necessarily binary) splits, a suitable generalization is the Gibbs fragmentation tree~\citep{MPW08}.

In this paper, we will work with joint distributions over both tree structure and data. These are typically inspired by, or based directly upon, the simpler tree-only distributions described above. Our primary focus is the Dirichlet diffusion tree~\citep{N03}, which is specified by a birth process that we will shortly describe. However, our methodology applies quite generally. Other notable Bayesian approaches to hierarchical clustering include: \citet{W00}, in which each node of the tree is annotated with a vector that is sampled from a Gaussian centered at its parent's vector; \citet{HG05}, that defines a distribution over flat clusterings and then specifies an agglomerative scheme for finding a good partition with respect to this distribution; \citet{AGJ08}, in which data points are allowed to reside at internal nodes of the tree; \citet{TDR08,BW12}, in which the distribution over trees is specified by a bottom-up coalescing process; and \citet{KG15}, which generalizes the Dirichlet diffusion trees to allow non-binary splits. 

\subsection{The Dirichlet diffusion tree}

The Dirichlet diffusion tree (DDT) is a 
generative model for $d$-dimensional vectors
$x_1, x_2, \ldots, x_N$. Data
are generated sequentially via a continuous-time
process, lasting from time $t = 0$ to $t = 1$,
whereupon they reach their final value.

The first point, $x_1$, is generated
via a Brownian motion, beginning at the origin, i.e.
$X_1(t + dt) = X_1(t) + \mathcal{N}(0, \sigma^2I_ddt)$
where $X_1(t)$ represents the value of $x_1$ at time $t$.
The next point, $x_2$, follows the path created
by $x_1$ until
it eventually \emph{diverges}
at a random time, according to a specified
acquisition function $a(t)$.
When $x_2$ diverges, it creates
an internal node in the tree structure
which contains both the time and value of $x_2$
when it diverged.
After divergence, it continues until $t = 1$ with 
an independent Brownian motion.
In general, the $i$-th point
follows the path created by the 
previous $i - 1$ points.
When it reaches a node, it
will first sample one of two branches
to enter, then
either 1) diverge on the branch,
whereupon a divergence time is sampled
according to the acquisition function $a(t)$,
or 2) recursively continue to the next node.
Each of these choices has a probability
associated with it, according to various properties
of the tree structure and choice of acquisition function
(details can be found in \citet{N03}).
Eventually, all points will diverge and continue independently,
creating an internal node storing the time and
intermediate value for each point at divergence.
The DDT thus defines a binary tree over the data
(see \autoref{fig:ddt} for an example).
Furthermore, given a DDT with $N$ points,
it is possible to sample the possible divergence
locations of a $(N + 1)$-th point, using
the generative process. 

\begin{figure}[h]
    \centering
    \includegraphics[width=0.45\textwidth]{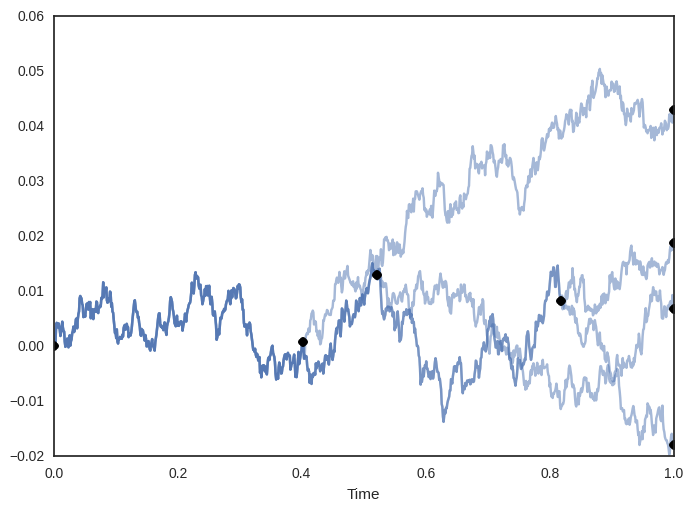}
    \caption{An example DDT with 1-dimensional data. Blue lines represent paths
    taken by each data point, and black dots represent nodes of the tree.
    The rightmost dots are leaves and the others are nodes
    created when points diverged. When drawing the hierarchy,
    typically the top stem is omitted.}
    \label{fig:ddt}
\end{figure}


Sampling the posterior
DDT given data can be done with the
Metropolis-Hastings (MH) algorithm,
an MCMC method.
The MH algorithm 
obtains samples from target distribution $p(x)$
indirectly by instead sampling
from a conditional ``proposal'' distribution $q(x | x')$,
creating a Markov chain whose stationary
distribution is $p(x)$, assuming $q$ satisfies some conditions.
Our choice of proposal distribution
modifies the DDT's tree structure
via a \emph{subtree-prune and regraft} (SPR) move,
which has the added benefit of extending to
other distributions over hierarchies.

\subsubsection*{The Subtree-Prune and Regraft Move}
\label{app:sprsampler}
An SPR move consists of first a \emph{prune} then a \emph{regraft}.
Suppose $T$ is a binary tree with $n$ leaves.
Let $s$ be a non-root node in $T$ 
selected uniformly at random and
$S$ be its corresponding subtree.
To prune $S$ from $T$, we 
remove $s$'s parent $p$
from $T$, and replace $p$ with $s$'s sibling.

Regrafting selects a branch at random
and attaches $S$ to it as follows. 
Let $(u, v)$ be the chosen branch ($u$ is the parent of $v$).
$S$ is attached to the branch by creating a node
$p$ with children $s$ and $v$ and parent $u$ 
(see \autoref{fig:sprmove} for an example).

\begin{figure*}[htp!]
    \centering
    \includegraphics[width=\textwidth]{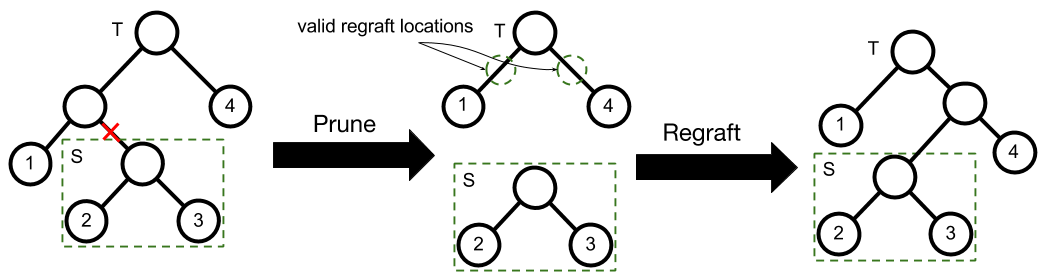}
    \caption{In the subtree-prune and regraft (SPR) move, a subtree $S$
        is selected uniformly at random and is then pruned from the tree. 
            Next, a regraft location is selected from the valid regraft locations, and $S$ is re-attached
            at that location.}
    \label{fig:sprmove}
\end{figure*}

The MH proposal distribution for the DDT
is an augmented SPR move, where
the time and intermediate value at each
node are sampled in addition to tree structure.
The exchangeability of the DDT enables
efficient sampling of regraft branches by 
simulating the generation process for a new point
and returning the branch and time where it diverges.
The intermediate values for the entire
tree are sampled via an interleaved Gibbs sampling move,
as all conditional distributions are Gaussian.

\section{Adding interaction}

Impressive as the Dirichlet diffusion tree is, there is no reason to suppose that it will magically find a tree that suits the user's needs. But a little interaction can be helpful in improving the outcome.

Let $T^*$ denote the target hierarchical clustering. It is not necessarily the case that the user would be able to write this down explicitly, but this is the tree that captures the distinctions he/she is able to make, or wants to make. Figure~\ref{fig:refinement} (left) shows an example, for a small data set of 5 points. In this case, the user does not wish to distinguish between points $1,2,3$, but does wish to place them in a cluster that excludes point $4$.

We could posit our goal as exactly recovering $T^*$. But in many cases, it is good enough to find a tree that captures all the distinctions within $T^*$ but also possibly has some extraneous distinctions, as in the right-hand side of Figure~\ref{fig:refinement}.

\begin{figure}
    \centering
    \includegraphics[width=3in]{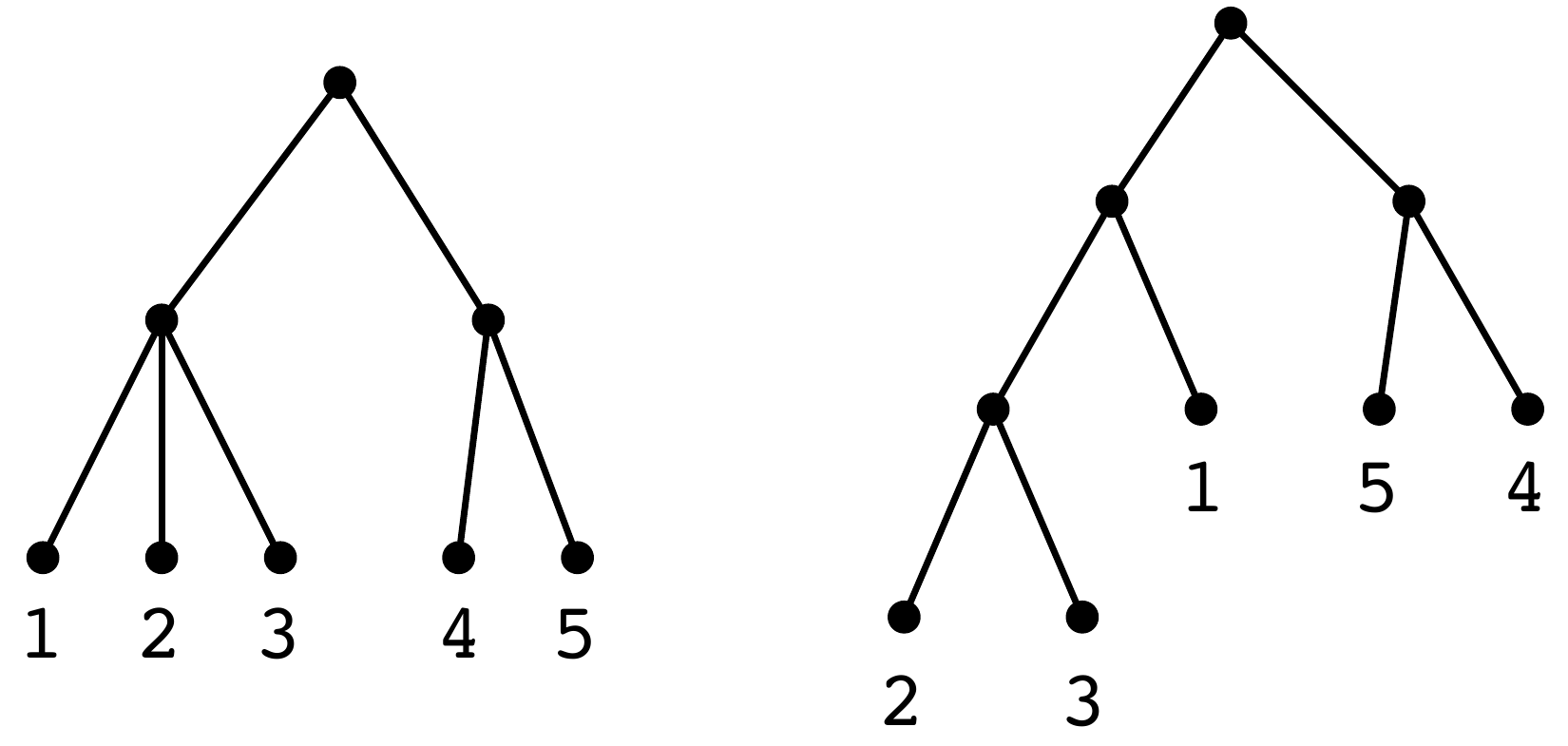}
    \caption{Target tree $T^*$ (left) and a refinement of it.}
    \label{fig:refinement}
\end{figure}

Formally, given data set $X$, we say $S \subseteq X$ is a {\it cluster} of tree $T$ if there is some node of $T$ whose descendant leaves are exactly $S$. We say $T$ is a {\it refinement} of $T^*$ if they have the same set of leaves, and moreover every cluster of $T^*$ is also a cluster of $T$. This, then, is our goal: to find a refinement of the target clustering $T^*$.

\subsection{Triplets}

The user provides feedback in the form of triplets. The constraint $(\{a,b\},c)$ means that the tree should have a cluster containing $a$ and $b$ but not $c$. Put differently, the lowest common ancestor of $a,b$ should be a strict descendant of the lowest common ancestor of $a,b,c$.

Let $\Delta(T)$ denote the set of all proper triplet constraints embodied in tree $T$. If $T$ has $n$ nodes, then $|\Delta(T)| \leq {n \choose 3}$. For non-binary trees, it will be smaller than this number. Figure~\ref{fig:refinement} (left), for instance, has no triplet involving $1,2,3$. 

Refinement can be characterized in terms of triplets.
\begin{restatable}{lemma}{treerefinement}
\label{thm:treerefinement}
Tree $T$ is a refinement of tree $T'$ if and only if $\Delta(T') \subseteq \Delta(T)$.
\end{restatable}
\begin{proof}
See supplement.
\end{proof}
In particular, {\it any triplet-querying scheme that converges to the full set of triplets of the target tree $T^*$ is also guaranteed to produce trees that converge to a refinement of $T^*$.}

With this lemma in mind, it is natural to measure how close a tree $T$ is to the target $T^*$ with the following (asymmetric) distance function, which we call {\it triplet distance} (TD):






\begin{align}
    \mathrm{TD}(T^*, T) = \frac{\sum_{c \in \Delta(T^*)} \mathbb{I}(c \notin \Delta(T)) }{|\Delta(T^*)|}
\end{align}
where $\mathbb{I}$ is the indicator function.
This distance is zero exactly when $T$ is a refinement of $T^*$, in which case we have reached our goal. 



A simple strategy for obtaining triplets 
would be to present the user with three randomly
chosen data points and have the user pick the odd one out.
This strategy has several drawbacks.
First, some sets of three points have no triplet constraint 
(for instance, points $1,2,3$ in \autoref{fig:refinement}).
Second, the chosen set of points might correspond to a triplet
that has already been specified, or is {\it implied} by 
specified triplets. For example, knowledge of $(\{a, b\}, c)$ and
$(\{b, c\}, d)$ implies $(\{a, c\}, d)$.
Enumerating the set of implied
triplets is non-trivial for $n > 3$ triplets~\citep{Bryant1995},
making it difficult to avoid these implied triplets
in the first place.

We thus consider another strategy---rather
than the user arranging three data points into
a triplet, the user observes
the hierarchy induced over some $O(1)$-sized
subset $S$ of the data and corrects
an error in the tree by supplying a triplet.
We call this is a \emph{subtree query}.
Finally, we note that in this work
we only consider the \emph{realizable} case
where the triplets obtained
from a user do not contain contradictory information
and that there is a tree that satisfies
all of them.

%


\subsection{Finding a tree consistent with constraints}
\label{sec:aho}
We start with a randomly initialized hierarchy $T$
over our data
and show an induced subtree $T|_S$ to the user, obtaining the
first triplet. The next step is constructing  
a new tree that satisfies the triplet.
This begins the feedback cycle; a user provides a triplet
given a subtree and the triplet is incorporated into a clustering algorithm,
producing a new candidate tree.
A starting point is 
an algorithm that returns 
a tree consistent with a set of triplets.

The simplest algorithm to solve this problem is
the \texttt{BUILD} algorithm, introduced in \citet{ASSU81}.
Given a set of triplets $C$, \texttt{BUILD}
will either return a tree that satisfies $C$, or error 
if no such tree exists.
In \texttt{BUILD}, we first construct the {\it Aho graph}
$G_C$, which has a vertex for each data point and an 
undirected edge $\{a,b\}$ for each triplet constraint $(\{a,b\},c)$.
If $G_C$ is connected, there is no tree that satisfies all
triplets. Otherwise, the top split of the tree is a partition of 
the connected components of $G_C$: any split is fine as long as 
points in the same component stay together. Satisfied triplets are discarded, and \texttt{BUILD} then continues recursively on the
left and right subtrees.

\texttt{BUILD} satisfies triplet constraints but ignores 
the geometry of the data, whereas we wish to take both into
account. By incorporating triplets into the posterior DDT 
sampler, we obtain high likelihood trees that still satisfy $C$.

\subsection{Incorporating triplets into the sampler}

In this section, we present an algorithm
to sample candidate trees 
from the posterior DDT, constrained
by a triplet set $C$.
It is based on the subtree prune and regraft move.

\begin{figure*}[htp!]
    \centering
    \includegraphics[width=\textwidth]{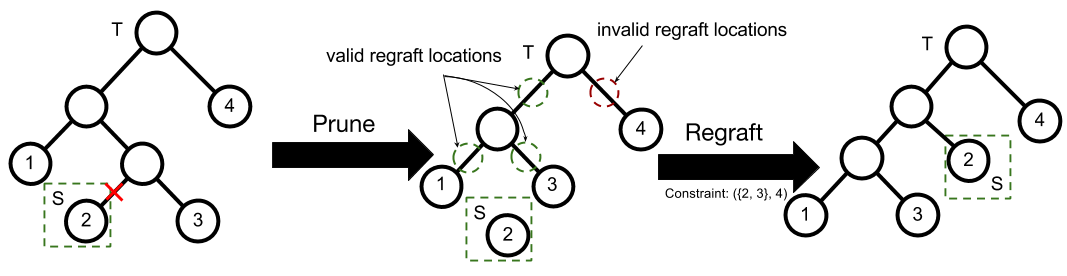}
    \caption{
            Visualized is a constrained-SPR move.
            Pruning is identical  but
            a regraft location is selected from the valid regraft locations
            limited by triplets.
            In this image, we are constrained by the sole triplet $(\{2, 3\}, 4)$.
            }
    \label{fig:constrainedsprmove}
\end{figure*}

The SPR move is of particular interest
because we can efficiently
enforce triplets to form a \emph{constrained-SPR move},
resulting in a sampler
that only produces trees that satisfy a set of triplets.
A \emph{constrained-SPR move} is defined as
an SPR move that assigns zero probability to
any resulting trees that would violate a set of triplets.
Restricting the neighborhood of an SPR move
runs the risk of partitioning the state space,
losing the convergence
guarantees of the Metropolis-Hastings algorithm.
Fortunately, a constrained-SPR move does not compromise
strong connectivity.
For any realizable triplet set $C$,
we prove 
the constrained-SPR move Markov chain's aperiodicity and irreducibility.

Consider the Markov chain on the state space 
of rooted binary trees that is induced by the constrained sampler.
\begin{restatable}{lemma}{aperiodic}
The constrained-SPR Markov chain is aperiodic.
\end{restatable}
\begin{proof}
A sufficient condition for aperiodicity
is the existent of a ``self-loop'' in the transition matrix: a non-zero probability of a state transitioning to itself.
Supposed we have pruned a subtree already.
When regrafting, the ordinary SPR move
has a non-zero probability of choosing any branch,
and a constrained-SPR move cannot regraft
to branches that would violate triplets.
Since the current tree in the Markov chain
satisfies triplet set $C$, there is a non-zero probability
of regrafting to the same location. 
We thus have an aperiodic Markov chain.
\end{proof}

\begin{restatable}{lemma}{irreducible}
\label{thm:irr}
A constrained-SPR Markov chain is irreducible.
\end{restatable}

\begin{proof}
(sketch) To show irreducibility,
we show that a tree $T$ has an non-zero
probability of reaching an arbitrary tree
$T'$ via constrained-SPR moves where
both $T$ and $T'$ satisfy a set of triplets $C$.
Our proof strategy
is to construct a canonical tree 
$T_C$, and show that there exists
a non-zero probability path from $T$ to $T_C$,
and therefore from $T'$ to $T_C$.
We then show that for a given constrained-SPR move,
the reverse move has a non-zero probability.
Thus, there exists a path from $T$ to $T_C$ to
$T'$, satisfying irreducibility.

Recall that the split at a node in a binary
tree that satisfies triplet set $C$
corresponds to a binary partition
of the Aho graph at the node (see Section \ref{sec:aho}).
$T_C$ is a tree such that
every node in $T_C$ is in \emph{canonical form}.
A node is in canonical form if it is a leaf node,
or, the partition of the Aho graph
at that node can be written as $(l, r)$.
$l$ is the single connected component
containing the point with the minimum data index,
and $r$ is the rest of the components.

To convert a particular node $s$ into canonical form,
we first perform ``grouping'',
which puts $l$ into a single descendant of $s$
via constrained-SPR moves.
We then make two constrained-SPR moves to convert
the partition at $s$ into the form $(l, r)$ (see \autoref{fig:canonical}).
We convert all nodes into canonical form recursively, turning 
an arbitrary tree $T$ into $T_C$.

Finally, the reverse constrained-SPR move has a non-zero
probability. Suppose
we perform a constrained-SPR move on tree $T_1$, converting it into $T_2$ by 
detaching subtree $s$
and attaching it to branch $(u, v)$.
A constrained-SPR move on $T_2$ can select
$s$ for pruning 
and can regraft it to form $T_1$ with a non-zero
probability since
$T_1$ satisfies the same constraint set as $C$.
For a full proof, please refer to the supplement.
\end{proof}

\begin{figure*}
    \centering
    \includegraphics[width=\textwidth]{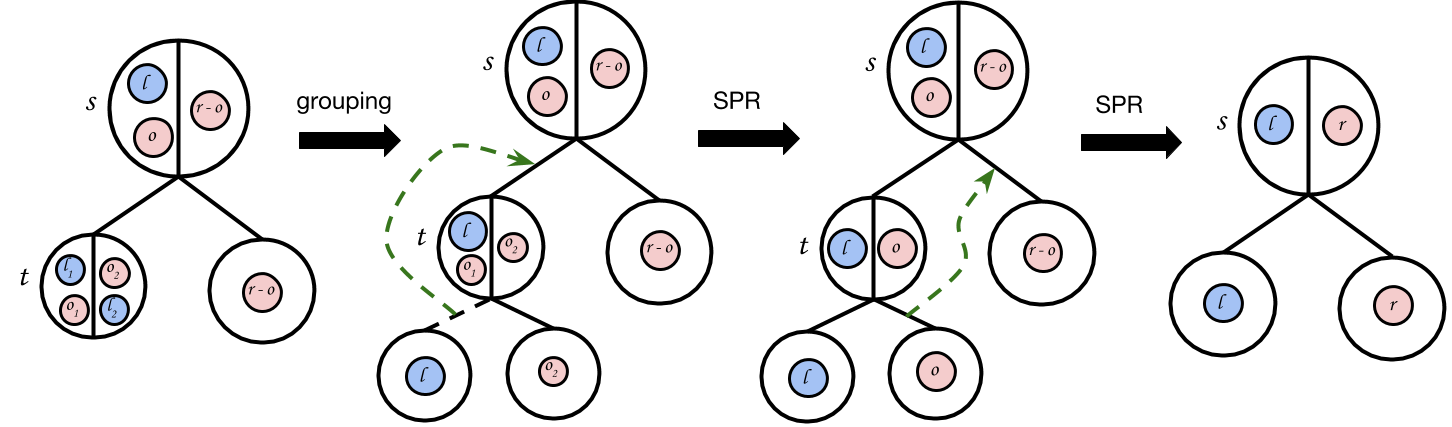}
    \caption{The process of converting $s$ into canonical form.
    We first group nodes from $l$ into their own isolated subtree, then perform
    two constrained-SPR moves to put $s$ into canonical form.}
    \label{fig:canonical}
\end{figure*}

%

The simplest possible scheme for a constrained-SPR move 
would be rejection sampling. The Metropolis-Hastings
algorithm for the DDT would be the same as in the
unconstrained case,
but any trees violating $C$ would have accept
probability $0$. Although this procedure is correct,
it is impractical. As the number of triplets
grows larger, more trees will be rejected
and the sampler will slow down over time.

To efficiently sample a tree that satisfies a set of triplets $C$, 
we modify the regraft in the ordinary SPR move. 
The constrained-SPR move must assign zero probability
to any regraft branches that would result in a tree
that violates $C$.
This is accomplished by generating the path from the root
in the same manner as sampling a branch,
but avoiding paths that would resulted in violated triplets.

\subsubsection*{Description of constrained-SPR sampler}

Recall that in the DDT's sampling procedure for regraft branches,
a particle at a node picks a branch, and either diverges from that
branch or recursively samples the node's child.
Let $s$ be the root of the subtree we are currently grafting 
back onto tree $T$, let $C$ be the
triplet set, and let $\texttt{leaves}(u)$
denote the descendant-leaves of node $u$.
Suppose we are are currently at node $u$,
deciding whether to diverge
at the branch $(u, v)$ or to recursively sample $v$.
Consider any triplet $(\{a, b\}, c) \in C$. If all---or none---of $a,b,c$ are in $\texttt{leaves}(s)$, then the triplet is 
unaffected by the graft, and can be ignored. Otherwise, 
some checks are needed:
\begin{enumerate}
\item $c \in \texttt{leaves}(s)$

Then we know $a,b \not\in \texttt{leaves}(s)$. If $a$ and $b$ are split across $v$'s children,
we are banned from sampling $v$.

\item $a \in \texttt{leaves}(s)$ but $b,c \not\in \texttt{leaves}(s)$

If both $b$ and $c$ are in $\texttt{leaves}(v)$, we are
required to sample $v$.
If just $c$ is in $\texttt{leaves}(v)$, we are banned
from both diverging at $(u, v)$ and sampling $v$.
Otherwise we can either diverge at $(u, v)$ or sample $v$.

\end{enumerate}

(The case where $b \in \texttt{leaves}(s)$ is
symmetric to case 2.)
If we choose to sample $v$, we remove
constraints from our current set $C$ that are now satisfied,
and continue recursively.
This defines a procedure by which we can sample
a divergence branch that does not violate constraints.

While the constrained-SPR sampler can produce
a set of trees given a set of static constraints,
the \texttt{BUILD} algorithm is useful in
adding new triplets into the sampler.
Suppose we have been sampling trees with constrained-SPR moves
with satisfying triplet set $C$ 
and we obtain a new triplet $u = (\{a, b\}, c)$ from a user query. 
We take the current tree $T$ and find the least common ancestor 
(call it $z$) of $a$ and $b$. We then call 
\texttt{BUILD}$(C + \{u\})$ on just the nodes in 
$\texttt{leaves}(z)$, and 
we substitute the resulting subtree at position $z$ in tree $T$.

\subsection{Intelligent subset queries}
We now have a method to sample a constrained
distribution over candidate trees.
Given a particular candidate tree $T$,
our first strategy for subtree querying
is to pick a random subset $S$ of the leaves
of constant size, and show the user
the induced subtree over the subset, $T|_S$.
We call this \emph{random subtree querying}.
But can we use a set of trees produced by the sampler
to make better subtree queries?
If tree structure is ambiguous in a particular region of data,
i.e. there are several hierarchies that could explain
a particular configuration of data, 
the MH algorithm will sample over
these different configurations. A query over points
in these ambiguous regions may help our algorithm
converge to a better tree faster. By looking for
these regions in our samples, we can choose query
subsets $S$ for which the tree structure is highly variable, 
and hopefully the resulting triplet from the user will 
reduce the ambiguity.

More precisely, we desire a notion of tree variance.
Given a set of trees $\mathcal{T}$, what is the variance
over a given subset of the data $S$?
We propose using the notion of tree distance
as a starting point. 
For a given tree $T$, the tree distance between two nodes $a$ and $b$,
denoted $\texttt{treedist}_T(a, b)$,
is the number of edges of $T$ needed to get from $a$ to $b$.
Consider two leaves $u$ and $v$. If the tree structure around
them is static, we expect the tree distance
between $u$ and $v$ to change very little, as the surrounding tree
will not change. However, if there is
ambiguity in the surrounding structure, the tree distance will
be more variable.
Given a subset of data $S$ and a set of trees $\mathcal{T}$,
the tree distance variance (TDV) of the trees over the subset is defined as:
\begin{align}
    \text{TDV}(\mathcal{T}, S) = \max_{u, v \in S}\mathrm{Var}_{T \in \mathcal{T}}[\texttt{treedist}_{T|_S}(u, v)]
\end{align}

This measure of variance is
the \emph{max} of the variance of tree distance between
any two points in the subset. Computing this requires
$O(|\mathcal{T}||S|^2 \log|S|)$ time, and since since $|S|$ is constant,
it is not prohibitively expensive.

Given a set of trees from the sampler $\mathcal{T}$,
we now select a high-variance subtree
by instantiating $L$ random subsets of constant size, $S_1, \ldots, S_L$
and picking $\text{argmax}_l \text{TDV}(\mathcal{T}, S_l)$.
We call this \emph{active subtree querying}.
Although using tree variance will help reconcile ambiguity in the tree structure,
if a set of samples from a tree all violate the same triplet,
it is unlikely that active querying will recover that triplet.
Thus, interleaving random querying and active querying
will hopefully help the algorithm converge quickly, while avoiding
local optima.

\section{Experiments}

\begin{figure*}
    \begin{subfigure}[]{0.5\textwidth}
        \centering
        \includegraphics[width=\textwidth]{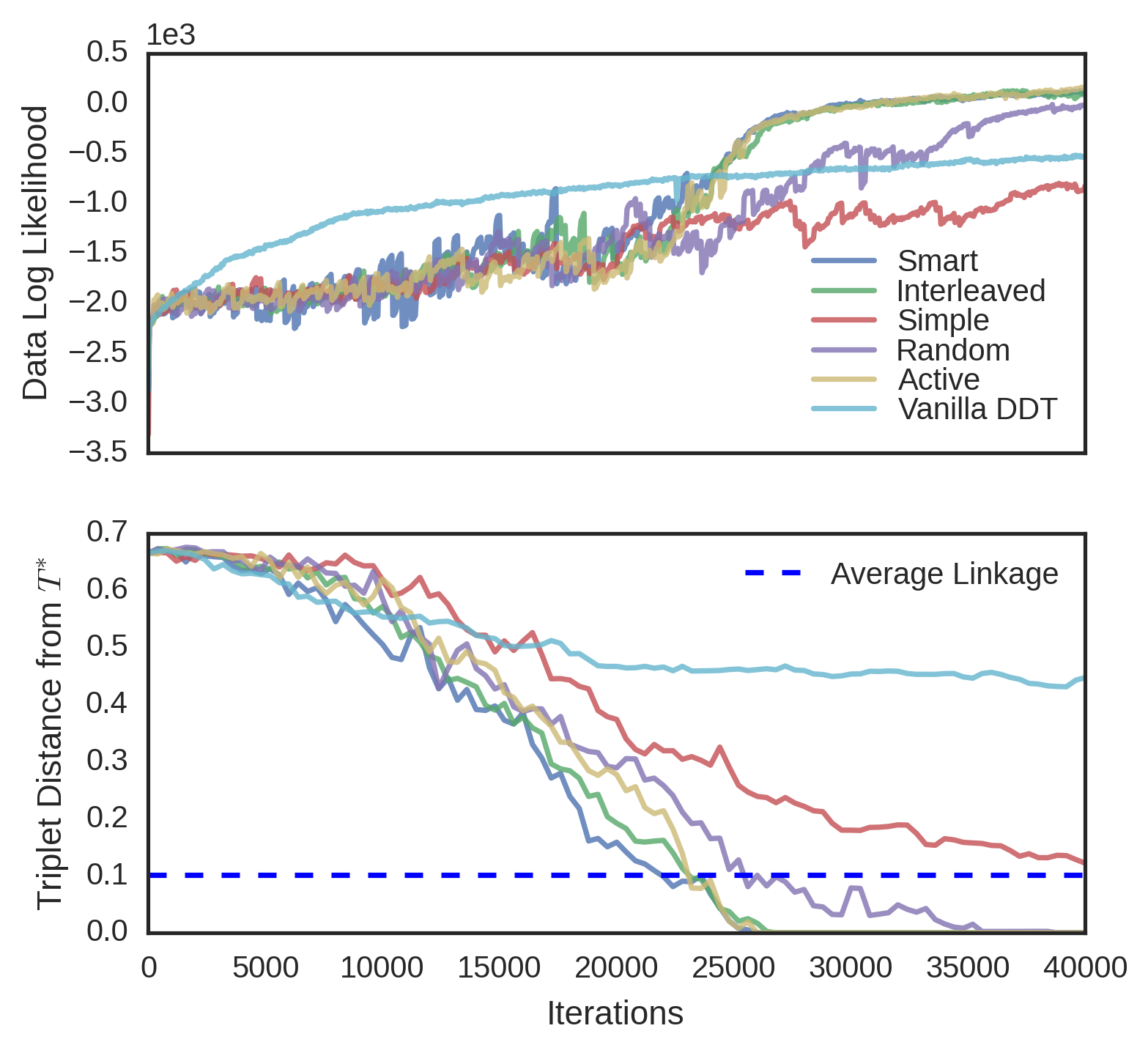}
        \caption{Fisher Iris}
        \label{fig:iris-result}
    \end{subfigure}
    \begin{subfigure}[]{0.5\textwidth}
        \centering
        \includegraphics[width=\textwidth]{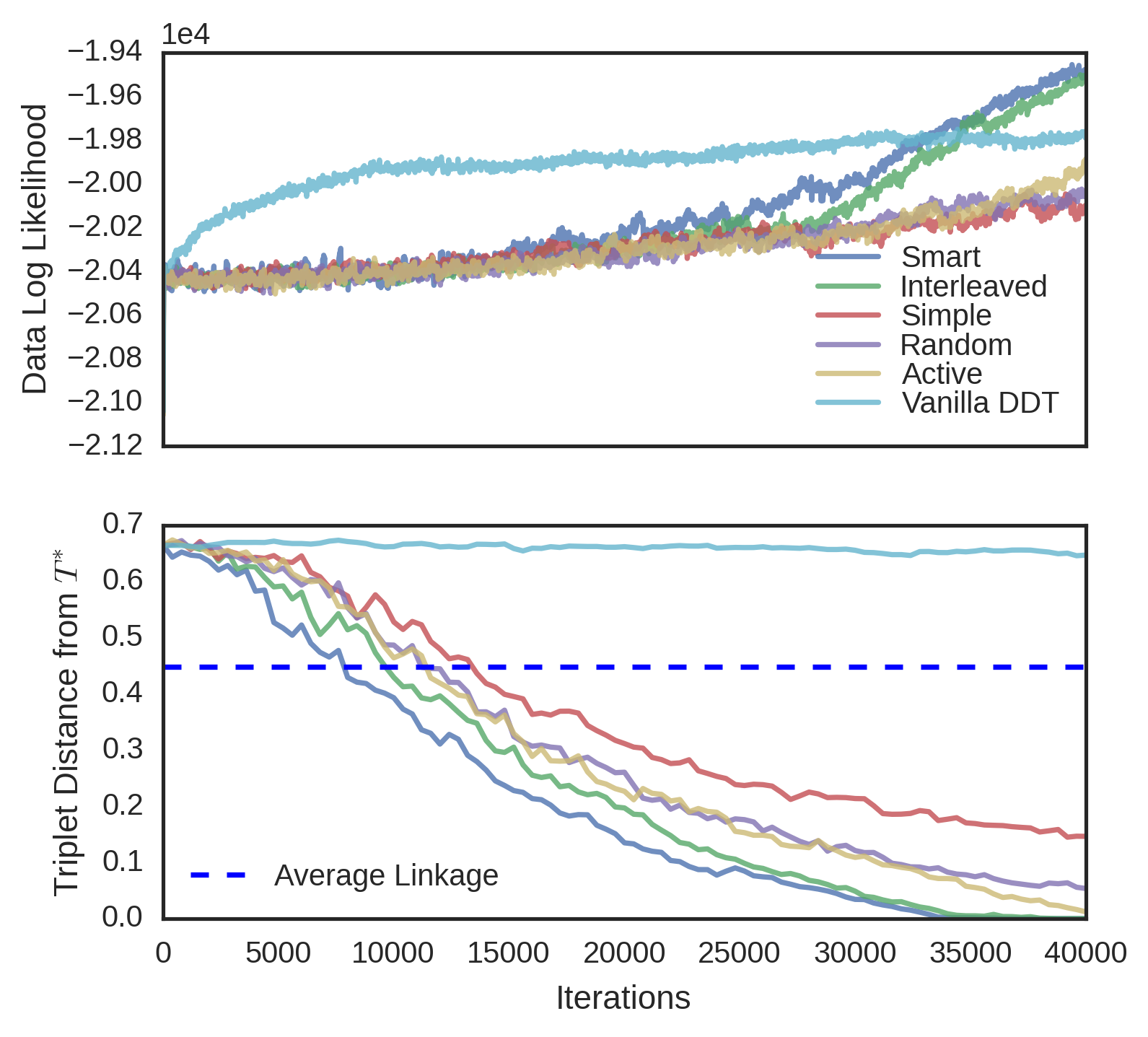}
        \caption{MNIST}
        \label{fig:mnist-result}
    \end{subfigure}
    \caption{The average of four runs of constrained-SPR samplers
    for the Fisher Iris dataset and the MNIST dataset, using 5 different querying schemes. A query was made every 100 iterations.}
    \label{fig:main-results}
\end{figure*}

We evaluated the convergence properties of five different querying schemes.
In a ``simple query'', a user is presented with three random data
and picks an odd one out.
In a ``smart query'', a user is unrealistically shown the entire candidate tree and reports a violated triplet.
In a ``random query'', the user is shown the induced candidate tree over
a random subset of the data.
In an ``active query'', the user is shown a high variance 
subtree using tree-distance variance.
Finally, in an ``interleaved query'', the user is alternatively shown a random
subtree and a high variance subtree.
In each experiment, $T^*$ was known, so user queries were simulated
by picking a triplet
violated by the root split of the queried tree, 
and if no such triplet existed, recursing on a child.
Each scheme was evaluated on four different datasets.
The first dataset, MNIST \citep{MNIST}, is an 10-way image classification
dataset where the data are 28 x 28 images of digits.
The target tree $T^*$ is simply the $K$-way classification 
tree over the data.
The second dataset is Fisher Iris, a 3-way flower
classification problem, where each of 150 flowers has
five features.
The third dataset, Zoo ~\citep{Lichman2013}, is a set of 93 animals
and 15 binary morphological features for each of animals,
the target tree being the induced binary tree 
from the Open Tree of Life ~\citep{Hinchliff2015}.
The fourth dataset is 20 Newsgroups ~\citep{20News}, a
corpus of text articles on 20 different subjects. We use
the first 10 principal components as features
in this classification problem.
All datasets were modeled with DDT's with
acquisition function $a(t) = 1/(1 - t)$
and Brownian motion parameter $\sigma^2$ estimated from data.
To better visualize the different convergence rates
of the querying schemes, MNIST and  20 Newsgroups were subsampled
to 150 random points.

For each dataset and querying scheme, we instantiated a SPR
sampler with no constraints. Every one hundred iterations of the 
sampler, we performed a query.
In subtree queries, we used subsets of size $|S| = 10$ 
and in active querying, the highest-variance subset was chosen from $L = 20$ different random subsets.
As baselines,
we measured the triplet distance of the vanilla DDT
and the average linkage tree.
Finally, results were averaged over four runs of each sampler.
The triplet distances for Fisher Iris and MNIST can be seen in \autoref{fig:main-results}. Results for the other datasets
can be found in the supplement.
Although unrealistic due to the size of the tree shown to the user, 
the smart query performed the best, achieving minimum error
with the least amount of queries. Interleaved followed next,
followed by active, random, and simple. In general, the vanilla
DDT performed the worst, and the average linkage
score varied on each dataset, but in all cases, the
subtree querying schemes performed better than both the vanilla DDT
and average linkage.

In three datasets (MNIST, Fisher Iris and Zoo), 
interactive methods
achieve higher data likelihood than the vanilla DDT.
Initially, the sampler is often restructuring the tree
with new triplets
and data likelihood is unlikely to rise. However, over time
as less triplets are reported,
the data likelihood increases rapidly.
We thus conjecture that triplet constraints 
may help the MH algorithm find better optima.

\section{Future Work}
We are interested in studying the non-realizable case, i.e.
when there does not exist a tree that satisfies triplet set $C$. We would also like to better understand the effect of constraints on searching
for optima using MCMC methods.

\newpage


\bibliography{paper}
\bibliographystyle{icml2016}

\appendix
\beginsupplement
\onecolumn

\section{Proof Details}
\label{sec:proof}

\treerefinement*
\begin{proof}
Suppose, first, that $T$ is a refinement of $T'$. Pick any triplet $(\{a,b\},c) \in \Delta(T')$. Then there is a node in $T'$ whose descendants include $a,b$ but not $c$. By the definition of refinement, $T$ contains a node with the same descendants. Hence the constraint $(\{a,b\},c)$ holds for $T$ as well.

Conversely, say $\Delta(T') \subseteq \Delta(T)$. Pick any cluster $S'$ of $T'$; it consists of the descendants of some node in $T'$. Consider the set of all triplet constraints consisting of two nodes of $S'$ and one node outside $S'$. Since these constraints also hold for $T$, it follows that the lowest common ancestor of $S'$ in $T$ must have exactly $S'$ as its set of descendants. Thus $S'$ is also a cluster of $T$.
\end{proof}

\irreducible*
\begin{proof}
To prove irreducibility, we show that there is a non-zero probability
of moving from state $T$ to $T'$, both of which satisfy $C$. We
accomplish this by first defining a \emph{canonical tree} $T_C$ given a triplet 
set $C$ and showing that we can reach $T_C$ from $T$ using  
constrained-SPR moves. We then show that for every constrained-SPR move,
there exists an equivalent reverse move that undoes it with non-zero
probability.
This proves that that from $T_C$ we can reach $T'$, 
creating a path from $T$ to $T_C$ to $T'$.

A binary tree $T$ can be entirely defined by the bipartitions
made over the data at each node. 
Let $G_n$ be the Aho graph for node $n$.
For a binary tree that satisfies a set of triplets, 
the split over the data at each node $n$ must 
be a bipartition of the connected components of $G_n$.
We define a particular node to be in \emph{canonical form}
if either a) it is a leaf, or b) the bipartition over $G_n$
at that node can be written as
$(l, r)$, where $l$ exactly matches a single, particular connected component of
$G_n$, and $r$ is the rest of the connected components. 
The particular component
$l$ is the connected component in $G_n$ with the minimum data index
inside it.
Note that we treat the children of nodes as unordered.
A canonical tree $T_C$ is one such that every node in the tree is
in canonical form.
To convert an arbitrary tree $T$ that satisfies $C$ into $T_C$, we first
convert the root node of $T$ into canonical form
using constrained-SPR moves.

Let $s$ be the root of $T$ and let $l$
be the set of points that ought to be in their own partition according
to $G_s$. In order for
$s$ not to be in canonical form, $l$ must be in a partition with 
data from other connected components in $G_s$, which we will call $o$.
The bipartition if $s$ were in canonical form would be $(l, r)$ and
the current non-canonical bipartition can thus be written as $(l + o, r - o)$.

We first examine $t$, the child of the root that contains $l + o$.
In general, the data from $l$ and the data from $o$ could
be split over the children of $t$, so the partition at $t$
can be written as $(l_1 + o_1, l_2 + o_2)$ where $l = l_1 + l_2$ 
and $o = o_1 + o_2$. This is visualized in the first
tree of \autoref{fig:canonical}.
We first group the data from $l$ into their own ``pure'' subtree
of $t$ as follows.
Let $u$ be the root of the
lowest non-pure subtree of $t$ that has
data from $l$ in both of its children.
There exist two subtrees that are descendants of $u$ 
that contain data from $l$ (one on the left and one on the right).
Those two subtrees
must be pure, and furthermore, they are both free to
move within $u$ via constrained-SPR moves because
they are in different connected components in $G_u$.
Thus, we can perform a constrained-SPR move to merge these
two pure subtrees together into a larger pure subtree.
We can repeat this process for $t$ until all nodes
from $l$ are in their own pure subtree of $t$.
The partition of $t$ can thus be written as
$(l + o_1, o_2)$, since the pure subtree may be several
levels down from $t$. This grouping process is visualized in
\autoref{fig:canonicalgrouping} and the results can be
seen in the second
tree in \autoref{fig:canonical}.

\begin{figure}
    \centering
    \includegraphics[width=0.5\textwidth]{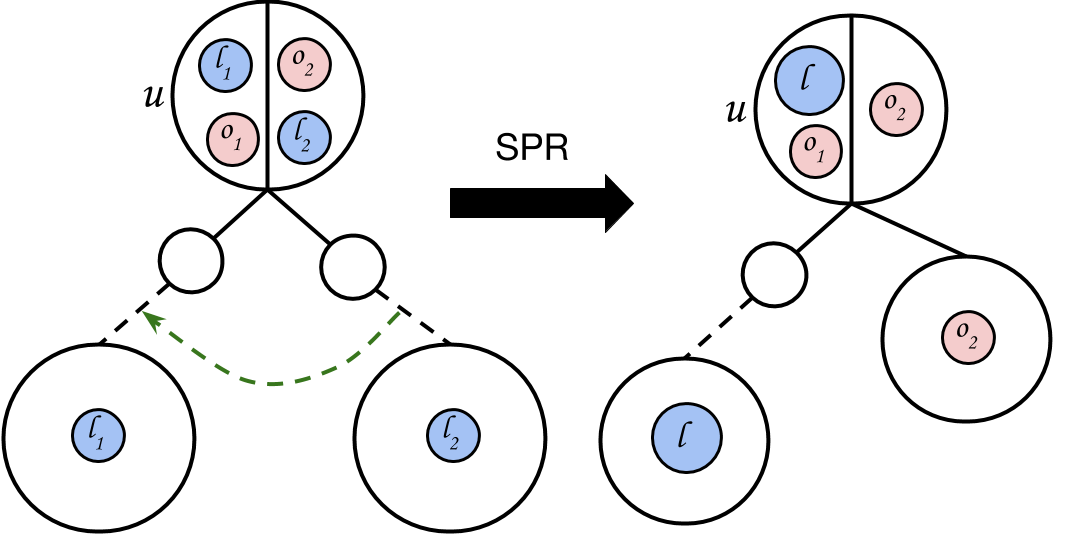}
    \caption{The process of grouping the data in $u$
    that belong to $l$ into their own pure subtree. $u$ is the
    lowest node of $t$ (see \autoref{fig:canonical}) that has data from $l$
    in both of its children.}
    \label{fig:canonicalgrouping}
\end{figure}

We now perform a constrained-SPR move to detach
the pure subtree of $l$ and regraft it to
the edge between $s$ and $t$. This is a permissible
move since $l$ is its own connected component in 
$G_s$. We now have the third tree in \autoref{fig:supcanonical}.
We now perform a final constrained-SPR
move, moving the subtree of $o$ to the opposite
side of $s$, creating the proper canonical partition
of $(l, r)$.
To entirely convert $T$ into $T_C$, we need to recurse
and convert every node in $T$ into canonical form.

\begin{figure*}
    \centering
    \includegraphics[width=\textwidth]{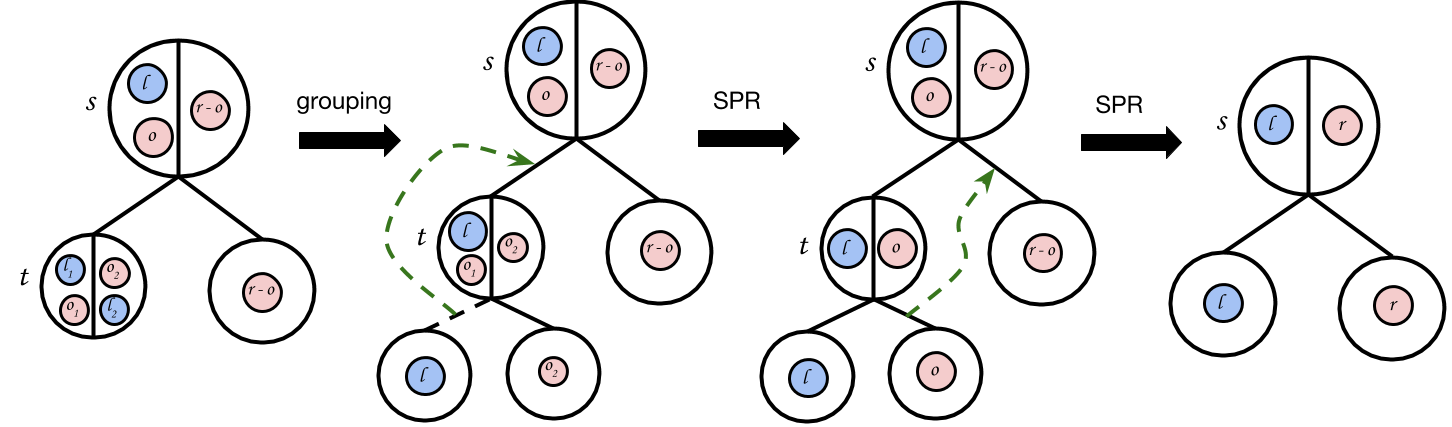}
    \caption{The process of converting $s$ into canonical form.
    We first group nodes from $l$ into their own pure subtree, then perform
    two constrained-SPR moves to put $s$ into canonical form.}
    \label{fig:supcanonical}
\end{figure*}

Every constrained-SPR move has an associated reverse constrained-SPR move that
performs the opposite transition.
The reverse constrained-SPR move selects the same subtree as the forward one
and prunes it, and just regrafts the subtree to its original location
before the forward move.
We know that this regraft has non-zero probability because the original tree
did not violate constraints.
Thus, since any arbitrary $T$ can be converted into $T_C$ and since each move
has a non-zero probability reverse move,
$T_C$ can be converted into an arbitrary tree $T'$ and we have a non-zero probability
path to convert $T$ into $T'$.

\end{proof}

\section{Additional Results}

\begin{figure*}[h]
    \begin{subfigure}[b]{0.5\textwidth}
        \centering
        \includegraphics[width=\textwidth]{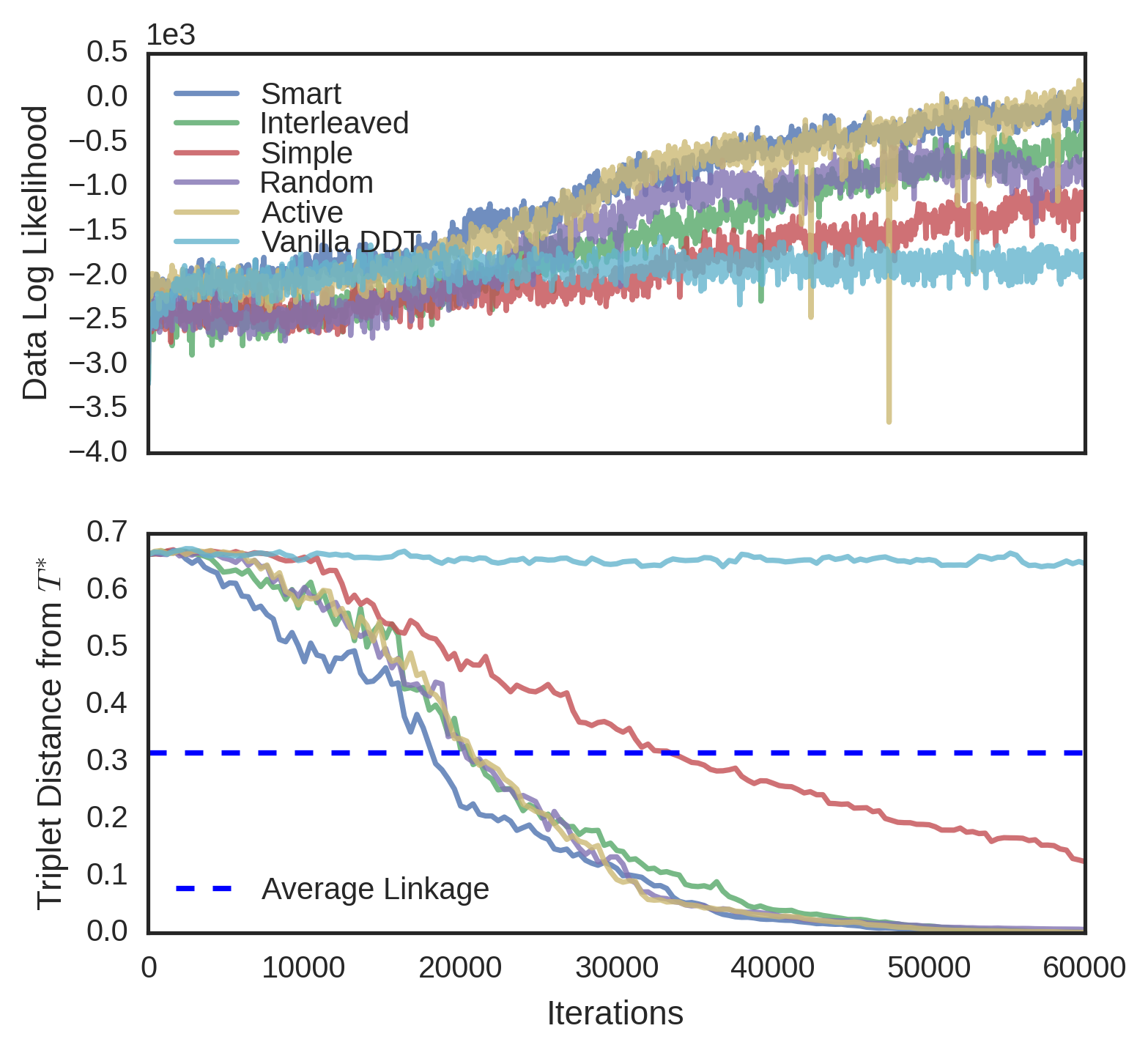}
        \caption{Zoo}
        \label{fig:zoo-result}
    \end{subfigure}
    \begin{subfigure}[b]{0.5\textwidth}
        \centering
        \includegraphics[width=\textwidth]{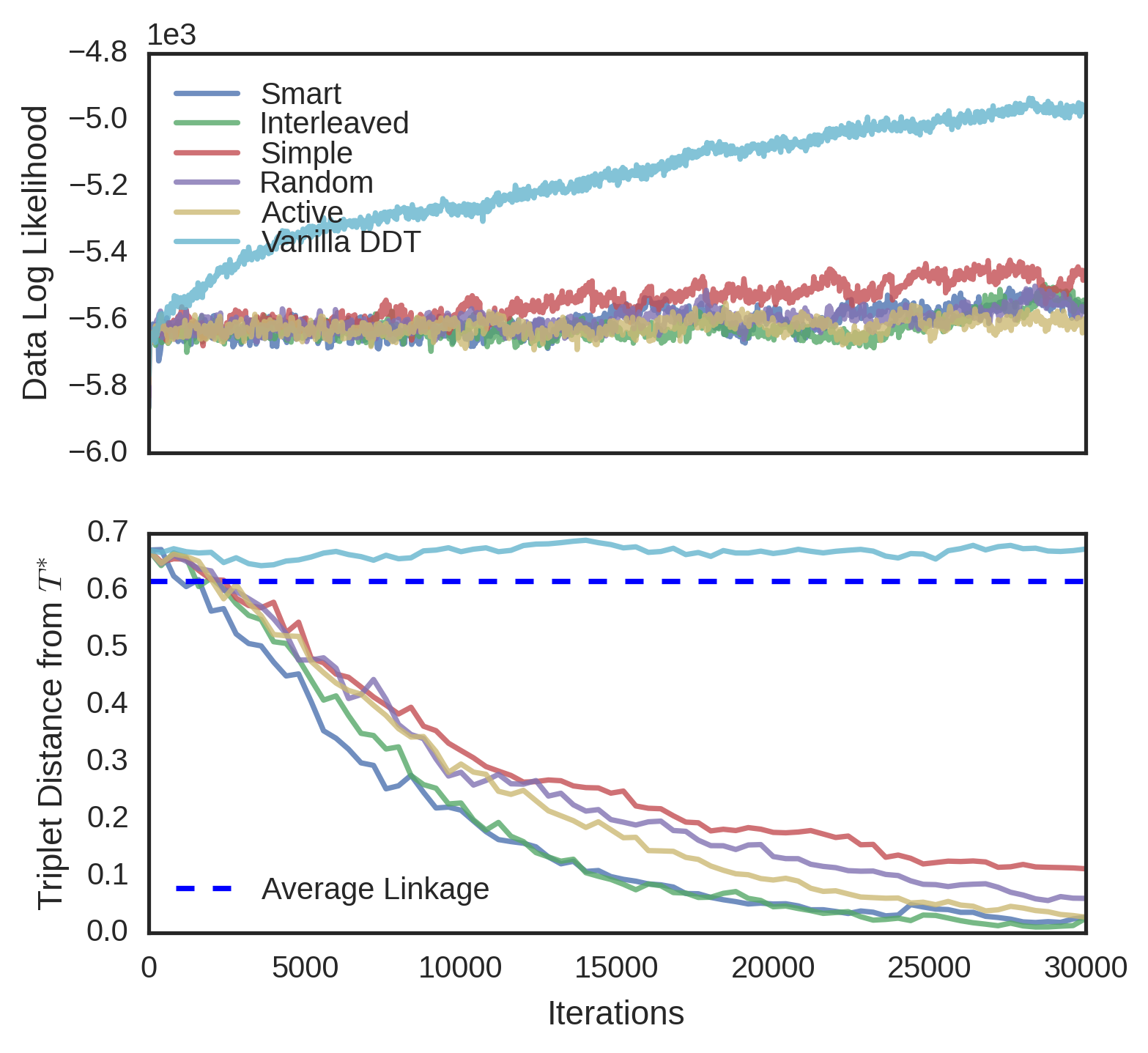}
        \caption{20 Newsgroups}
        \label{fig:20news-result}
    \end{subfigure}
    \caption{The average of four runs of constrained-SPR samplers
        for the Zoo dataset and the 20 Newsgroups dataset, using 5 different querying schemes.
        A query was made every 100 iterations.}
    \label{fig:main-results2}
\end{figure*}

\end{document}